\documentclass[letterpaper, 10 pt, journal, twoside]{IEEEtran}
\IEEEoverridecommandlockouts
\usepackage{cite}
\usepackage{amsmath,amssymb,amsfonts}
\usepackage{bm}
\usepackage{algorithmic}
\usepackage{graphicx}
\usepackage{textcomp}
\usepackage{xcolor,color}
\usepackage{amsthm}
\usepackage[ruled,vlined]{algorithm2e}
\usepackage[top=0.833in, left=0.667in, right=0.667in, bottom=0.597in, includefoot]{geometry}
\usepackage{hyperref}

\usepackage{pifont}

\makeatletter
\let\MYcaption\@makecaption
\makeatother 

\usepackage[font=footnotesize]{subcaption}

\makeatletter
\let\@makecaption\MYcaption
\makeatother

\newtheorem{definition}{Definition}

\newtheorem{proposition}{Proposition}
\newtheorem{remark}{Remark}
\newtheorem{problem}{Problem}
\newtheorem{assumption}{Assumption}

\DeclareMathOperator{\until}{\mathbf{U}}
\DeclareMathOperator{\always}{\mathbf{G}}
\DeclareMathOperator{\eventually}{\mathbf{F}}

\newcommand{\x}{\mathbf{x}}
\renewcommand{\u}{\mathbf{u}}   
\newcommand{\y}{\mathbf{y}}
\newcommand{\q}{\mathbf{q}}
\newcommand{\M}{\mathbf{M}}
\newcommand{\C}{\mathbf{C}}
\newcommand{\btau}{\bm{\tau}}

\begin{document}

\title{Trajectory Optimization for High-Dimensional Nonlinear Systems under STL Specifications}

\author{Vince Kurtz and Hai Lin}

\maketitle
\thispagestyle{empty}

\begin{abstract}
    Signal Temporal Logic (STL) has gained popularity in recent years as a specification language for cyber-physical systems, especially in robotics. Beyond being expressive and easy to understand, STL is appealing because the synthesis problem---generating a trajectory that satisfies a given specification---can be formulated as a trajectory optimization problem. Unfortunately, the associated cost function is nonsmooth and non-convex. As a result, existing synthesis methods scale poorly to high-dimensional nonlinear systems. In this letter, we present a new trajectory optimization approach for STL synthesis based on Differential Dynamic Programming (DDP). It is well known that DDP scales well to extremely high-dimensional nonlinear systems like robotic quadrupeds and humanoids: we show that these advantages can be harnessed for STL synthesis. We prove the soundness of our proposed approach, demonstrate order-of-magnitude speed improvements over the state-of-the-art on several benchmark problems, and demonstrate the scalability of our approach to the full nonlinear dynamics of a 7 degree-of-freedom robot arm.   
\end{abstract}

\section{Introduction and Related Work}\label{sec:intro}

A key challenge in Cyber-Physical Systems (CPS) research is accomplishing complex tasks with systems that include nontrivial physical dynamics. The fact that the system is required to perform complex tasks limits the usefulness of standard control techniques for setpoint tracking and classical notions of stability. At the same time, accounting for high-dimensional physical dynamics limits the applicability of many standard planning techniques from the computer science literature. 

Much recent progress towards the goal of effective CPS control has been made through the use of temporal logic \cite{lin2014mission}. Temporal logics like Linear Temporal Logic (LTL), Computation Tree Logic (CTL), and Signal Temporal Logic (STL), encode complex system behaviors as logical formulas. These logics are both expressive, in the sense that many complex behaviors can be encoded as temporal logic formulas, and interpretable, in the sense that most formulas can be easily understood by human readers. 

In this letter, we focus our attention on STL in particular. STL is defined over continuous-valued signals, making it especially well-suited to CPS with high-dimensional physical dynamics, where output trajectories provide convenient signals. Furthermore, STL is appealing because the synthesis problem can be formulated as a trajectory optimization problem \cite{belta2019formal}. This is possible through the use of {quantitative} (or ``robust'') semantics, which map a signal to a real-valued scalar.

\begin{figure}
    \centering
    \includegraphics[width=0.8\linewidth]{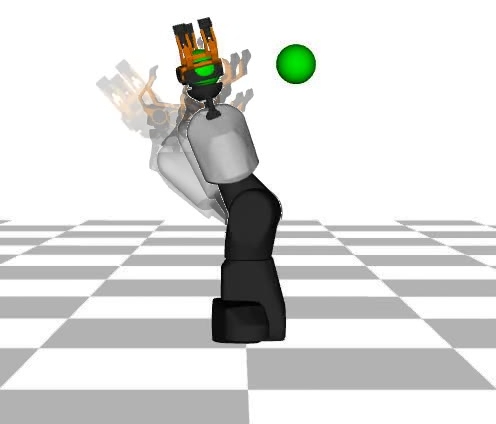}
    \caption{We present a scalable trajectory optimization method for high-dimensional nonlinear systems that uses Differential Dynamic Programming to achieve Signal Temporal Logic specifications. In this example, the full nonlinear dynamics of a 7-DoF robot manipulator are used to compute a trajectory such that the end-effector reaches either one of the two green targets.}
    \label{fig:manipulator_1}
\end{figure}

Unfortunately, the resulting optimization problem is both non-smooth and non-convex. Exact solutions can be obtained using Mixed Integer Linear Programming (MILP) \cite{sadraddini2015robust,belta2019formal,raman2014model}, but these results are limited to linear systems and suffer from scalability issues. Specifically, each additional timestep and predicate requires the addition of a new integer variable, meaning that the MILP approach is exponential in the specification length, specification complexity, and (often) system dimension. 

To overcome these scalability challenges, and to extend STL trajectory optimization methods to nonlinear systems, there has been a recent focus on smooth approximations of the STL {quantitative} semantics{\cite{pant2017smooth,mehdipour2019arithmetic,gilpin2021smooth,varnai2020robustness}}. The resulting optimization {over ``robustness values''} is still nonconvex, but its smoothness enables the use of gradient-based techniques like Sequential Quadratic Programming (SQP). {Furthermore, custom data structures can enable efficient evaluation of gradients for complex STL specifications \cite{leung2020back}}. These {gradient-based} methods significantly outperform MILP on large problems and can be applied to nonlinear systems. But they are susceptible to local minima, especially for high-dimensional systems with very complex nonlinear dynamics. 

In this letter, we propose an STL trajectory optimization approach for high-dimensional nonlinear systems based on Differential Dynamic Programming (DDP) \cite{mayne1966second}. DDP is a trajectory optimization approach that relies on iterative linear approximations of the system dynamics and quadratic approximations of the optimal cost-to-go. In this way, DDP is similar in spirit to SQP, but takes advantage of the structure of the trajectory optimization problem. As a result of these structural insights, DDP tends to be less vulnerable to local minima than SQP and other gradient-based methods \cite{tassa2012synthesis}, and has demonstrated remarkable effectiveness on very high-dimensional, highly nonlinear systems like robotic quadrupeds and humanoids \cite{tassa2012synthesis,tassa2014control,li2020hybrid,koenemann2015whole,mastalli20crocoddyl}.

We show how, for an expressive fragment of STL, DDP can be used to solve the synthesis problem. {Specifically, we introduce a systematic framework for designing cost functions associated with STL specifications, while also maintaining structural properties (smoothness, running cost) necessary for advanced trajectory optimization methods like DDP.} Our simulation results indicate a roughly order-of-magnitude improvement in solve time over gradient-based methods. Furthermore, this approach scales well to highly nonlinear systems. We demonstrate this scalability in simulation experiments with the full nonlinear dynamics of a 7 degree-of-freedom (DoF) robot arm. To the best of our knowledge, this is the first time that STL planning methods have been applied to the full nonlinear dynamics of a high-DoF torque-controlled manipulator. 

The remainder of this letter is organized as follows. Section \ref{sec:background} provides a formal problem statement and gives necessary background on STL and DDP. Section \ref{sec:main_results} presents our main result: an algorithm for using DDP to solve the STL synthesis problem. Simulation results are presented in Section \ref{sec:simulation}, and we conclude with Section \ref{sec:conclusion}.

\section{Background}\label{sec:background}

\subsection{System Definitions}

We assume that the physical dynamics of a CPS can be written as a set of nonlinear difference equations as follows:
\begin{equation}\label{eq:system}
\begin{gathered}
    \x_{t+1} = f(\x_t, \u_t) \\
    \y_t = g(\x_t, \u_t),
\end{gathered}
\end{equation}
where $\x_t \in \mathbb{R}^n$ is the system state, $\u_t \in \mathbb{R}^m$ is a control input, and $\y_t \in \mathbb{R}^p$ is the system output. We assume that $f$ and $g$ may be nonlinear, but are continuously differentiable. 

\subsection{Problem Formulation}

Given an STL specification, our primary task is to find a system trajectory that satisfies the specification:

\begin{problem}\label{prob:main}
    Given an STL specification $\varphi$, an initial condition $\x_0$, and system dynamics of the form (\ref{eq:system}), determine a control sequence $U = \u_0,\u_1,\dots,\u_T$ such that the resulting output signal $Y = \y_0,\y_1,\dots,\y_T$ satisfies the specification, $Y \vDash \varphi$.
\end{problem}

\subsection{Signal Temporal Logic}\label{subsec:stl}

In this section, we provide a brief overview of STL, its {quantitative} semantics, and smooth approximations thereof. For further details, we refer the interested reader to \cite{belta2019formal,pant2017smooth,mehdipour2019arithmetic,gilpin2021smooth,sadraddini2015robust} and related work therein.

Signal temporal logic is defined over continuous-valued signals $Y = \y_0,\y_1,\y_2,\dots$, which in our case represent output trajectories. We focus in this work on discrete-time signals, though STL can also be considered over continuous-time signals \cite{haghighi2019control}. 

The syntax of STL is given by 
\begin{equation}\label{eq:syntax}
    \varphi := \pi \mid \lnot \varphi \mid \varphi_1 \land \varphi_2 \mid \varphi_1 \until_{[t_1,t_2]} \varphi_2,
\end{equation}
where $\pi = (\mu^\pi(\y_t) \geq 0)$ is a predicate defined by the function $\mu^{\pi} : \mathbb{R}^p \to \mathbb{R}$. Negation ($\lnot$) and conjuction ($\land$) operators can be used to derive other boolean operators like disjuction ($\lor$) and implication ($\implies$). Similarly, the temporal operator until ($\varphi_1\until_{[t_1,t_2]}\varphi_2)$ can be used to construct eventually ($\eventually_{[t_1,t_2]}\varphi$) and always ($\always_{[t_1,t_2]}\varphi$) operators. In this work, we assume bounded-time STL formulas, i.e., {there exists a time $T$ at which the satisfaction status is fixed for any given signal}. 

We present the semantics, or meaning, of an STL formula using STL's {quantitative} semantics. These semantics associate a scalar value, called the ``robustness measure'', with a given signal $Y$. The signal satisfies the specification if the associated robustness measure is positive {and violates the specification if the robustness measure is negative. If the robustness measure is zero, satisfaction is undefined. For an in-depth discussion of these robust semantics and a comparison with standard STL semantics, we refer the interested reader to \cite{belta2019formal}}. Given a signal $Y$, let $(Y,t) = \y_t,\y_{t+1},\dots$ denote the suffix starting at timestep $t$. We can then define the STL {quantitative} semantics recursively as follows:

\begin{definition}[STL {Quantitative} Semantics]~\label{def:robust_semantics}
\begin{itemize}
    \item {$\rho^\varphi((Y,0)) > 0 \implies Y \vDash \varphi$}
    \item $\rho^\pi( (Y,t) ) = \mu^\pi(\y_t)$
    \item $\rho^{\lnot \varphi}( (Y,t) ) = - \rho^\varphi( (Y,t) )$
    \item $\rho^{\varphi_1 \land \varphi_2}( (Y,t) ) = \min\big(\rho^{\varphi_1}( (Y,t) ), \rho^{\varphi_2}( (Y,t) ) \big)$
    \item $\rho^{\eventually_{[t_1,t_2]} \varphi}( (Y,t) ) = \max_{t'\in[t+t_1, t+t_2]}\big(\rho^\varphi( (Y,t') )\big)$
    \item $\rho^{\always_{[t_1,t_2]} \varphi}( (Y,t) ) = \min_{t'\in[t+t_1, t+t_2]}\big(\rho^\varphi( (Y,t') )\big)$
    \item $\rho^{\varphi_1 \until_{[t_1,t_2]} \varphi_2}( (Y,t) ) =  \max_{t'\in[t+t_1, t+t_2]}\bigg( \\ \min\Big(\Big[\rho^{\varphi_2}( (Y,t') ), \min_{t'' \in[t+t_1,t')}\big(\rho^{\varphi_1}( (Y,t'') )\big)\Big]^T\Big)\bigg)$.
\end{itemize}
\end{definition}
Taking some liberty of notation, we write $\rho^\varphi(Y) = \rho^\varphi((Y,0))$.
{We note that some presentations of STL semantics define the until ($\until$) operator such that $\varphi_1$ must hold from $t$ to $t''$ \cite{raman2014model}. Following \cite{belta2019formal}, we enforce $\varphi_1$ only from $t+t_1$ to $t''$. }

As discussed above, the {quantitative} semantics lead naturally to a trajectory optimization formulation of Problem \ref{prob:main}:
\begin{align}\label{eq:original_optimization}
    \min_{U} & -\rho^\varphi(Y) \\
    \text{s.t. } & \x_{t+1} = f(\x_t,\u_t) \\
                 & \y_{t} = g(\x_t,\u_t) \\
                 & \x_0 \text{ given.}
\end{align}

A major challenge in solving this optimization problem comes from the non-smoothness of $\rho^\varphi$, which is induced by the use of $\min$ and $\max$ operators in the {quantitative} semantics. A smooth approximation of the robustness measure, denoted $\tilde{\rho}^\varphi$, can thus be provided by simply taking smooth approximations of the $\min$ and $\max$ operators \cite{pant2017smooth,gilpin2021smooth,varnai2020robustness}. 

A particularly appealing choice of smooth approximations are
\begin{equation}\label{eq:ourmin}
    \widetilde{\min}(a_1,\dots,a_m) = -\frac{1}{k_1}\log \Big(\sum_{i=1}^{m} e^{-k_1 a_{i}} \Big),
\end{equation}
 \begin{equation}\label{eq:ourmax}
    \widetilde{\max}(a_{1},\dots,a_{m}) = \frac{\sum_{i=1}^{m} a_{i} e^{k_2 a_{i}}}{\sum_{i=1}^{m} e^{k_2 a_{i}}},
\end{equation}
where $k_1,k_2 > 0$ are tuneable parameters. This particular choice is useful because both approximations are strict underapproximations ($\widetilde{\min} < \min$, $\widetilde{\max} < \max$), but as $k_1, k_2 \to \infty$, they approach the original operators. This means that {for STL specifications written in Postive Normal Form (PNF) \cite{raman2014model}}, if the resulting smooth robustness $\tilde{\rho}^{\varphi}$ is positive, then the signal must satisfy the specification. For further details on this approach, we refer the interested reader to \cite{gilpin2021smooth}. 

By replacing $\rho^\varphi$ with the smooth approximation $\tilde{\rho}^\varphi$, the optimization (\ref{eq:original_optimization}) becomes smooth, and can be solved with gradient-based techniques like SQP. 

\subsection{Differential Dynamic Programming}\label{subsec:ddp}

DDP is a trajectory optimization method for solving problems of the following form:
\begin{align}\label{eq:ddp_optimization}
    \min_{U} & \sum_{t=0}^T l_t(\x_t,\u_t) \\
    \text{s.t. } & \x_{t+1} = f(\x_t,\u_t) \\
                 & \x_0 \text{ given,}
\end{align}
where $f$ is a smooth, possibly nonlinear function as above, and the running cost $l$ is twice continuously differentiable in $\x_t$ and $\u_t$. This running cost need not be convex. 

The basic idea behind DDP is to iteratively update a nominal trajectory based on a local linear approximation of the dynamics and a local quadratic approximation of the optimal cost-to-go. This process can be summarized as follows:
\begin{enumerate}
    \item Guess a nominal control sequence $U = \u_0,\u_1,\dots,\u_T$.
    \item Forward pass: simulate the system forward, computing a linear approximation of the dynamics. 
    \item Backwards pass: starting from time $T$, use the linearized dynamics to compute a quadratic approximation of the optimal cost-to-go. 
    \item Update $U$ accordingly, typically via a line search. 
    \item Using the updated control sequence, repeat from (2). 
\end{enumerate}

For further details on this process, we refer the interested reader to \cite{mayne1966second,tassa2014control,li2004iterative}. DDP is similar in spirit to SQP, but takes advantage of structure of the trajectory optimization problem (\ref{eq:ddp_optimization}). Besides faster solve times, these structural insights lead to reduced sensitivity to local minima. Furthermore, the quadratic approximation of the optimal cost-to-go enables a full-state feedback controller about the nominal trajectory, at no extra computational cost. 

The empirical effectiveness of DDP for highly nonlinear systems has been well-demonstrated in the robotics literature: DDP has been used to control quadrupeds \cite{li2020hybrid,grandia2019feedback}, humanoids \cite{tassa2014control,koenemann2015whole}, and other high-DoF robots \cite{tassa2012synthesis,mastalli20crocoddyl}. 

\section{Differential Dynamic Programming for Signal Temporal Logic}\label{sec:main_results}

In this section, we demonstrate how the advantages of DDP---speed and scalability to high-dimensional nonlinear dynamics---can be brought to the STL synthesis problem. 

The main challenge in using DDP for STL synthesis is in formulating a smooth running cost like (\ref{eq:ddp_optimization}). As discussed above, the STL robustness measure is nonsmooth. Furthermore, the robustness measure cannot, in general, be written as a sum of running costs. For example, consider the specification $\eventually_{[0,5]} \varphi_1$. The robustness measure associated with $t=5$ depends on whether $\varphi_1$ has been satisfied at a previous timestep, so a simple sum of running costs is not sufficient. 

Smoothness is relatively easy to address: we simply use a smooth approximation of the robustness \cite{gilpin2021smooth}, $\tilde{\rho}^{\varphi}$. As noted in Section \ref{sec:background}, $\tilde{\rho}^{\varphi} < \rho^{\varphi}$ {for $\varphi$ in PNF}, and as $k_i \to \infty$, $\tilde{\rho}^{\varphi} \to \rho^\varphi$. 

To formulate a running cost, we need to make several simplifying assumptions. First, we consider a fragment of STL:
\begin{assumption}\label{as:fragment}
    Assume that the STL specification $\varphi$ can be written as follows:
    \begin{gather}
        \psi := \pi \mid \lnot \pi \mid \psi_1 \land \psi_2 \mid \psi_1 \lor \psi_2, \\
        \varphi := \always_{[t_1,t_2]} \psi \mid \eventually_{[t_1,t_2]} \psi \mid \psi_1 \until_{[t_1,t_2]} \psi_2 \mid \varphi_1 \land \varphi_2,
    \end{gather}
    where $\psi$ are \textit{state formulas} and $\varphi$ are \textit{path (STL) formulas}. 
\end{assumption}
{Note that this fragment is in PNF, i.e., negations are only applied at the predicate level. This ensures that $\tilde{\rho}^\varphi$ is a strict underapproximation of $\rho^\varphi$.} The primary difference between this fragment and full STL is that disjunctions are only permitted between state formulas, and nested temporal operators are excluded. {Note that this fragment is more expressive than the STL fragment considered
in many recent works \cite{lindemann2017prescribed,lindemann2018control,varnai2019prescribed}, as disjuctions are allowed between state formulas}.

The next simplifying assumption is inspired by an implicit assumption in \cite{lindemann2018control}, where control barrier functions enforce the timesteps at which particular predicates hold: 

\begin{assumption}\label{as:switching}
    Switching times for temporal operators are fixed. 
\end{assumption}
{For example, if $ \varphi =\eventually_{[t_1,t_2]} \psi$, we assume that $\psi$ holds at $t=t_2$. Similarly, if $ \varphi = \psi_1 \until_{[t_1,t_2]} \psi_2$, we assume that $\psi_2$ holds at $t=t_2$ and $\psi_1$ holds for all $t \in [t_1,t_2)$. This is slightly less restrictive than \cite{lindemann2018control}, since we impose no constraints on the signal other than at the switching time.}

{Any particular choice of switching times, whether chosen explicitly or by selecting a control barrier function, limits completeness. For example, if the switching times are fixed at $t_2$ as suggested above, the specification $(\eventually_{[0,10]} y > 1)\land(\eventually_{[0,10]} y < 0)$ is unsatisfiable. But with different choices of switching times, the problem may be solved. For the remainder of this letter, we assume that switching times are fixed at the end of the interval, but other valid choices could also be made.}

These simplifying assumptions allow us to formulate a running cost. The basic idea is to associate each timestep $t$ with zero or more state formulas, each of which depend only on $\y_t = g(\x_t,\u_t)$, and not on any future or past states. This process is formally described in Algorithm \ref{alg:ddp_stl}.

\begin{algorithm}
    \SetKwInOut{Input}{Input}
    \SetKwInOut{Output}{Output}
    
    \Input{STL specification $\varphi$; initial condition $\x_0$, parameters $k_1, k_2$; initial guess $U$}
    \Output{Satisfying trajectory $(X^*,U^*)$ or \textit{No Solution}}
    
    Initialize running cost $l_t(\x,\u) = null$ for $t \in [0,T]$
    
    \ForAll{path formula $\varphi_i$ in $\varphi$}{
        \If{ $\varphi_i = \always_{[t_1,t_2]}\psi$ }{
            $l_t(\x_t,\u_t) = \widetilde{\max}\{l_t, -\tilde{\rho}^\psi(g(\x_t,\u_t))\}$
            for all $t \in [t_1,t_2]$
        }
        \ElseIf{ $\varphi_i = \eventually_{[t_1,t_2]}\psi$ }{
            $l_t(\x_t,\u_t) = \widetilde{\max}\{l_t, -(t_2-t_1)\tilde{\rho}^\psi(g(\x_t,\u_t))\}$
            for $t = t_2$
        }
        \ElseIf{ $\varphi_i = \psi_1\until_{[t_1,t_2]}\psi_2$ }{
            $l_t(\x_t,\u_t) = \widetilde{\max}\{l_t, -\tilde{\rho}^\psi_1(g(\x_t,\u_t))\}$ for $t \in [t_1,t_2)$
            
            $l_t(\x_t,\u_t) = \widetilde{\max}\{l_t, -(t_2-t_1)\tilde{\rho}^\psi_2(g(\x_t,\u_t))\}$
            for $t = t_2$
        }
    }
    
    If $l_t(\x,\u) = null$ for any $t$, let $l_t(\x,\u) = 0$. 
    
    $X^*,U^* = DDP(\x_0,U)$
    
    \If{ $l_t(\x^*_t,\u^*_t) < 0 ~~\forall~ t$ }{
        \Return $X^*,U^*$
    }
    \Return \textit{No Solution}
    
    \caption{DDP-based STL Synthesis}
    \label{alg:ddp_stl}
\end{algorithm}

We start by initializing the running cost at each timestep with an empty value. Recall that (by Assumption \ref{as:fragment}) the specification can be written as $\varphi = \varphi_1 \land \varphi_2 \land \dots \land \varphi_N$. We iterate over each of these path formulas to set the running cost. Each $\varphi_i$ consists of a temporal operator (always, eventually, or until) acting over state formulas $\psi$. For each $\varphi_i$, we update the running cost using $\tilde{\rho}^\psi$, the smooth robustness measure associated with the state formula $\psi$. {This $\tilde{\rho}^\psi$ depends only on $\x_t$ and $\u_t$, not on any future or past states or control inputs.}

There are several features to mention here. First, Assumption \ref{as:switching} allows us to {associate} each state formula with a unique set of timesteps. There may be several state formulas associated with a given timestep: in this case, we combine them with the smooth maximum. This ensures that both state formulas must be satisfied for $l_t(\x_t,\u_t) < 0$. Additionally, we weight the cost associated with the until and eventually operators by the length of the interval $t_2-t_1$. This heuristic ensures that there is adequate ``incentive'' to satisfy these subformulas, since the cost is only associated with a single time instant. 

Given such a smooth running cost, we can simply write the STL trajectory optimization problem in the form of (\ref{eq:ddp_optimization}) and use DDP to find a (locally) optimal trajectory $(X^*, U^*)$. Finally, if the running cost $l_t(\x^*_t,\u^*_t)$ is negative at each timestep, we conclude that $\tilde{\rho}^{\varphi}(Y^*) > 0 \implies \rho^\varphi(Y^*) > 0 \implies Y^* \vDash \varphi$. 

Algorithm \ref{alg:ddp_stl} is sound, as shown in the following Proposition:

\begin{proposition}[Soundness]
    Given a specification $\varphi$, a system (\ref{eq:system}), and an initial condition $\x_0$, if Algorithm \ref{alg:ddp_stl} returns a solution $(X^*, U^*)$, then the associated output trajectory $Y^*$ satisfies the STL specification $\varphi$.
\end{proposition}
\begin{proof}
    Note that the existence of a solution implies $l_t(\x^*_t,\u^*_t) < 0$ for all $t$. Recalling that $\varphi$ can be written as the conjuction of path formulas, we consider each temporal operator separately: 
    
    Case 1: $\varphi_i = \always_{[t_1,t_2]}\psi$. By construction, $0 < -l_t(\x^*_t,\u^*_t) < \rho^\psi(\y_t)$ for all $t \in [t_1,t_2]$, and so $Y^*$ satisfies $\varphi_i$.
    
    Case 2: $\varphi_i = \eventually_{[t_1,t_2]}\psi$. By construction, $0 < -l_t(\x^*_t,\u^*_t) < \rho^\psi(\y_t)$ for $t=t_2$, and so $Y^*$ satisfies $\varphi_i$.
    
    Case 3: $\varphi_i = \psi_1\until_{[t_1,t_2]}\psi_2$. By construction, $0 < -l_t(\x^*_t,\u^*_t) < \rho^{\psi_2}(\y_t)$ for $t=t_2$, and $0 < -l_t(\x^*_t,\u^*_t) < \rho^{\psi_1}(\y_t)$ for all $t \in [t_1,t_2)$. Therefor $Y^*$ satisfies $\varphi_i$. 
    
    Now we have that $\varphi = \bigwedge_{i} \varphi_i$ and $Y^* \vDash \varphi_i ~\forall i$, so $Y^* \vDash \varphi$ and the theorem holds. 
\end{proof}

This soundness property means that if Algorithm \ref{alg:ddp_stl} returns a solution, this solution solves Problem \ref{prob:main}. It should be noted, however, that Assumption \ref{as:switching} and the use of DDP lead to some conservatism, so this approach is not complete: if Algorithm \ref{alg:ddp_stl} does not return a solution, Problem \ref{prob:main} may still have a solution. {In practice, Algorithm \ref{alg:ddp_stl} is less likely to find a solution for STL formulas with many local minima, and for STL formulas that depend heavily on the choice of switching times, such as the example following Assumption \ref{as:switching}.}

\begin{remark}
    If Algorithm \ref{alg:ddp_stl} returns no solution, there are several possibilities for attempting to find a satisfying trajectory. The simplest step would be re-run Algorithm \ref{alg:ddp_stl} with a different initial guess, $U$. While DDP tends to be less susceptible than other gradient-based methods to local minima, a new initial guess can enable convergence to a viable solution. Another possibility is to increase the value of the parameters $k_1$, $k_2$. This ensures that the smooth robustness measure $\tilde{\rho}^\psi$ is a closer approximation of the true robustness measure, and can enable satisfaction. {Of course, the closeness of this approximation must be balanced with the poorer numerical conditioning associated with higher $k_i$}. If the solution from the previous (failed) iteration of Algorithm \ref{alg:ddp_stl} is used as an initial guess, however, {similar methods for iterative approximation of contact constrains in the manipulation literature \cite{onol2019contact} suggest that numerical issues are less likely in practice}. 
\end{remark}

Finally, as discussed in Section \ref{subsec:ddp} above, DDP provides a local full-state feedback controller along the nominal trajectory at no extra computational cost, by virtue of locally approximating the optimal cost-to-go. This feedback is especially useful for CPS in the real world, which are inevitably subject to a variety of external disturbances and modeling errors.

{
\begin{remark}
    The complexity of one iteration of DDP is $O(Tm^3)$ \cite{tassa2008receding}. The number of iterations required for convergence depends on the initial guess and the particular problem at hand, but DDP typically terminates within about 100 timesteps. So the complexity of our approach is roughly linear in the STL time bound $T$.
\end{remark}}

\section{Simulation Results}\label{sec:simulation}

In this section, we present simulation experiments supporting our main results. Except where otherwise noted, all implementations were in python, with gradients and Hessians computed numerically. All experiments were performed on a laptop with an i7 processor and 32GB RAM. 

\subsection{Mobile Robot Motion Planning}

We first consider a standard baseline example for STL trajectory optimization: motion planning for a simple mobile robot. We assume that the robot's dynamics are given by a simple single integrator,
\begin{gather*}
    \x_{t+1} = \x_t + \u_tdt, \quad
    \y_t = \x_t,
\end{gather*}
where $\x_t \in \mathbb{R}^2$ is the robot's position and $dt=0.01s$ is the sampling period.

We first consider one of the simplest robot motion planning problems: reaching a goal while avoiding an obstacle. This specification can be written as
\begin{equation}\label{eq:reach_avoid}
    \varphi = \always_{[0,100]} \lnot \psi_{obstacle} \land \eventually_{[0,100]} \psi_{goal},
\end{equation}
where $\psi_{obstacle}$ and $\psi_{goal}$ are state formulas representing entering the obstacle and goal regions respectively. 

We compare our approach with \cite{gilpin2021smooth} as a representative approach for optimizing over a smoothed cost function with SQP. For our approach, we use a python implementation of \cite{tassa2014control} to solve the DDP optimization. For the baseline SQP approach, we use Scipy \cite{2020SciPy-NMeth} and computed gradients symbolically with autograd \cite{maclaurin2015autograd}. This gives some advantage to the baseline method (Scipy includes bindings to compiled C code, and analytical gradients are available). A random control sequence ($\u_t$ sampled from a uniform distribution in $[-1,1]$) was used to initialize both methods. {$k_1=k_2=10$ was used for all experiments. }

\begin{figure}
    \begin{subfigure}{0.49\linewidth}
        \centering
        \includegraphics[width=\linewidth]{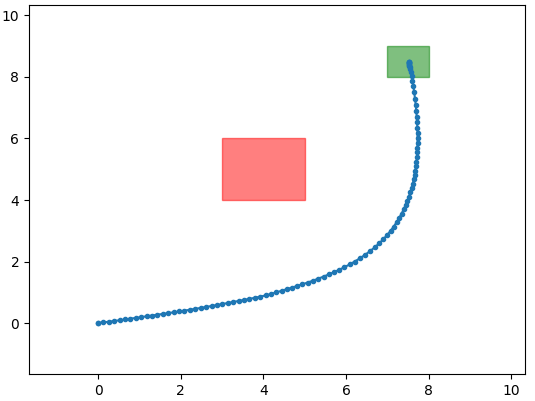}
        \caption{SQP approach \cite{gilpin2021smooth}. Optimization time 22.4s}
        \label{fig:reach_avoid_sqp}
    \end{subfigure}
    \begin{subfigure}{0.49\linewidth}
        \centering
        \includegraphics[width=\linewidth]{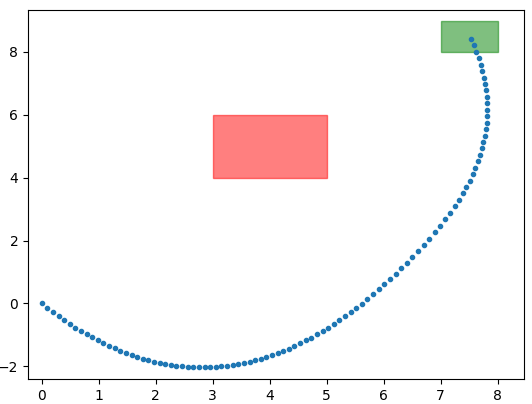}
        \caption{Our approach (DDP). Optimization time 2.3s}
        \label{fig:reach_avoid_ddp}
    \end{subfigure}
    \caption{A mobile robot (blue dot) must navigate around an obstacle (red) and reach a goal region (green), as encoded in the STL specification (\ref{eq:reach_avoid}). We compare our proposed approach to a standard smooth optimization method. Both methods find a satisfying trajectory, but our proposed approach is much faster.}
    \label{fig:reach_avoid}
\end{figure}

The results are shown in Figure \ref{fig:reach_avoid}. Both methods find a satisfying solution. Despite the fact that the baseline SQP approach has access to gradients that are computed analytically, and our approach uses a numerical computation of gradients, our approach outperforms the baseline in terms of computation time by roughly an order of magnitude (2.3s vs. 22.4s). 

Of course, this type of reach-avoid specification could also be achieved using a variety of motion planning techniques from the robotics literature like Rapidly-expanding Random Trees and Probabilistic Roadmap algorithms \cite{thrun2002probabilistic}. This motivates us to consider a more complex motion planning example that includes logical decision making elements. 

The scenario shown in Figure \ref{fig:either_or} represents something of a benchmark problem in the STL planning literature \cite{belta2019formal,mehdipour2019arithmetic,gilpin2021smooth}. A robot must visit one of two intermediate targets (blue) and avoid collisions with an obstacle (red) before reaching a goal region (green). This specification can be written formally as 
\begin{equation}\label{eq:either_or}
    \varphi = (\lnot \psi_{obstacle} \until_{[0,50]} \psi_{goal}) \land (\eventually_{[0,33]} \psi_{target~1} \lor \psi_{target~2}).
\end{equation}

Again, we compare our proposed approach with a baseline smooth optimization method \cite{gilpin2021smooth} which uses SQP. Both resulting trajectories, shown in Figure \ref{fig:either_or}, satisfy the specification, but our approach finds a satisfying solution faster (2.2s vs 16.6s). 

\begin{figure}
    \begin{subfigure}{0.49\linewidth}
        \centering
        \includegraphics[width=\linewidth]{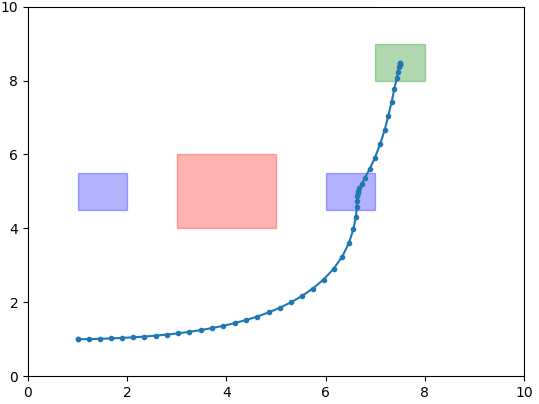}
        \caption{SQP approach \cite{gilpin2021smooth}. Optimization time 16.6s}
        \label{fig:either_or_sqp}
    \end{subfigure}
    \begin{subfigure}{0.49\linewidth}
        \centering
        \includegraphics[width=\linewidth]{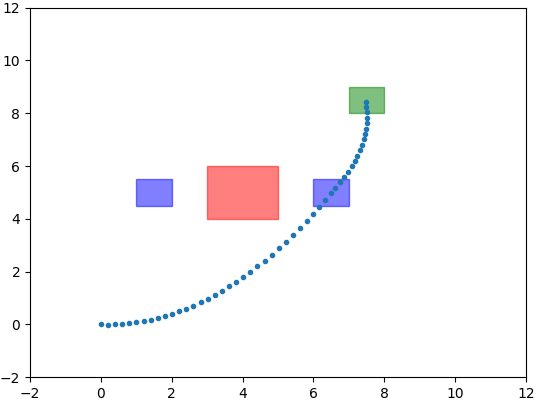}
        \caption{Our approach (DDP). Optimization time 2.2s}
        \label{fig:either_or_ddp}
    \end{subfigure}
    \caption{A mobile robot must navigate around an obstacle (red) and reach one of two target regions (blue) before arriving at a goal (green), as encoded in specification (\ref{eq:either_or}). Again, both our approach and the baseline find satisfying solutions, but our DDP-based approach finds a solution significantly faster than the baseline SQP approach.}
    \label{fig:either_or}
\end{figure}

\subsection{High-DoF Manipulator Control}

The experiments above demonstrate that our proposed method outperforms state-of-the-art gradient-based methods on several benchmark examples with simple dynamics. The primary advantage of our proposed approach, however, is scalability to high-dimensional nonlinear dynamics. 

To this end, we present an example of optimizing over the full nonlinear dynamics of a 7-DoF manipulator. Specifically, we use the Toro Arm model supplied with Crocoddyl \cite{mastalli20crocoddyl}. 

The manipulator is governed by the following dynamics,
\begin{equation}\label{eq:manipulator_dynamics}
    \M(\q)\ddot{\q} + \C(\q,\dot{\q})\dot{\q} + \btau_g = \btau,
\end{equation}
where $\q$ are joint angles, $\M$ is the positive definite mass matrix, the term $\C\dot{\q}$ collects Coriolis and centripetal terms, $\btau_g$ are torques due to gravity, and $\btau$ are control torques. We consider the state to be $\x_t = [\q^T ~ \dot{\q}^T]^T$, the output to be joint angles $\y_t = \q$, and the control input $\u_t = \btau$ to be the torques applied at each joint. We formulate a difference equation of the form (\ref{eq:system}) by using a forward-euler integration scheme at 200Hz. 

For all dynamics computations, we use the open-source library Crocoddyl \cite{mastalli20crocoddyl}, which provides python bindings to underlying C++ code. For trajectory optimization, we use the same python implementation of \cite{tassa2014control} as above, again computing gradients and Hessians numerically. 

The scenario we consider is illustrated in Figure \ref{fig:manipulator_1}. Essentially, we want the robot's end effector to reach one of two target positions, denoted $\mathbf{p}_1$ and $\mathbf{p}_2$, and shown in green in the figures. We first solve an inverse kinematics problem numerically to compute joint angles associated with reaching each target, $\q^{nom}_1$ and $\q^{nom}_2$. We then write a specification as follows:
\begin{equation}\label{eq:manipulator_spec}
    \varphi = \always_{[40,50]} (\|\q-\q^{nom}_1\| \leq \delta) \lor (\|\q-\q^{nom}_2\| \leq \delta) 
\end{equation}

This specification ensures that we must reach a state $\q$ within some distance $\delta = 0.01$ of one of the target joint angles within 40 timesteps, and remain there for 10 timesteps. 

An initial guess of $\u_t = \btau_g$ was used to compensate for gravitational effects. We applied Algorithm \ref{alg:ddp_stl} from two initial conditions, starting on either side of the two targets. In both cases, the DDP solver terminated successfully after about 160s. We suspect that a significant portion of the solve time is dominated by numerically computing gradients for such a high-dimensional system. {In contrast, the SQP method ran for over 5 hours without converging to a valid solution. }


\section{Conclusion}\label{sec:conclusion}

We presented a method of STL-based trajectory optimization for high dimensional nonlinear systems based on differential dynamic programming. We proved the soundness of our approach, and demonstrated scalability to the full nonlinear dynamics of a 7-DoF robot arm. Additionally, simulation experiments on a benchmark problems with simpler dynamics indicate a roughly order-of-magnitude speed improvement over state-of-the-art gradient-based methods. Future work will focus on applying these methods to more complex tasks, other complex CPS like legged robots, and for model predictive control. {Extensions to model predictive control schemes are of particular interest, since optimal solutions from previous iterations could be used to warm-start the solver, potentially leading to significant improvements in computation time.}

\bibliographystyle{unsrt}
\bibliography{references}

\end{document}